\documentclass{article}
\usepackage[preprint]{neurips_2026}

\usepackage[utf8]{inputenc}
\usepackage[T1]{fontenc}
\usepackage{amsmath, amssymb, amsfonts}
\usepackage{amsthm}
\usepackage{mathtools}
\usepackage{booktabs}          
\usepackage{enumitem}          
\usepackage{hyperref}          
\usepackage{thmtools}          
\usepackage{xcolor}            
\usepackage{algorithm}
\usepackage{algpseudocode}
\usepackage{cleveref}

\newcommand{\applabel}[1]{\label[appendix]{#1}}

\algrenewcommand\algorithmicrequire{\textbf{Inputs:}}
\algrenewcommand\algorithmicensure{\textbf{Output:}}

\newtheorem{definition}{Definition}[section]
\newtheorem{theorem}[definition]{Theorem}
\newtheorem{lemma}[definition]{Lemma}
\newtheorem{observation}{Observation}
\newtheorem{corollary}[definition]{Corollary}

\theoremstyle{remark}

\newcommand{\jinx}{q^*}

\newcommand{\Distr}[1]{\mathsf{Distr}\!\left(#1\right)} 

\newcommand{\pp}{\mathbf{p}}
\newcommand{\avgp}[1]{\overline{p}_{#1}}     
\newcommand{\E}{\mathbb{E}}

\newcommand{\LM}{{\sf LM}} 

\DeclareMathOperator{\med}{med}
\newcommand{\eps}{\varepsilon}  

\newcommand{\defeq}{\coloneqq}

\Crefname{claim}{Claim}{Claims}
\Crefname{algorithm}{Alg.}{Algs.}
\Crefname{appendix}{Appendix}{Appendices}

\title{Consensus Sampling for Safer Generative AI}

\author{%
Adam Tauman Kalai\\
OpenAI\\
\texttt{adam@kal.ai}
\And
Yael Tauman Kalai\\
MIT\\
\texttt{tauman@mit.edu}
\AND
Or Zamir\\
Tel Aviv University\\
\texttt{orzamir@tauex.tau.ac.il}
}

\begin{document}
\maketitle

\begin{abstract}
Motivated by undetectable risks in generative AI outputs, we study a general robust aggregation problem: how to aggregate several probability distributions to boost safety. We present consensus sampling, a black-box algorithm that, given $k$ distributions, either produces a sample or abstains. It's risk is competitive with that of the safest models. This yields an architecture-agnostic approach to generative-model safety when the distributions are induced by models that can sample and evaluate output probabilities. We formalize the guarantee through $R$-robustness, which also bounds information leakage and adversarial influence. Inspired by robust statistics and the provable copyright protection algorithm of \citet{vyas2023provable}, we show that while a standard mixture is vulnerable to one unsafe constituent, a pointwise-median construction provides robust intuition, and our efficient sampler is Pareto-optimal for the tradeoff between worst-case risk and abstention. Experiments on synthetic distributions and image generation illustrate the general mechanism and its motivating safety application. The method requires overlap among safe distributions, but it provides a model-agnostic way to inherit guarantees from an unknown reliable subset.\end{abstract}

\section{Introduction}

The rapid adoption of language models ($\LM$s) and generative AI has created unprecedented opportunities alongside serious safety challenges. Misaligned systems pose escalating privacy and security risks. They may, for instance, generate code with exploitable vulnerabilities or embed harmful secrets in ordinary text or images through \textit{steganography}, both of which can be provably undetectable through inspection alone  \citep{cohen87virus,witt2023perfectly,zamir2024undetectable,motwani2024secret,perry2025robuststeganographylargelanguage}. 
Even defining ``safe outputs'' remains a formidable challenge.

Motivated by this safety problem, we study a more general robust-sampling question: given several probability distributions, how can one aggregate these distributions in a way that is resilient to unsafe or even adversarial constituent distributions? Generative models provide one important application because each model induces a distribution over outputs, but the mathematical problem is about aggregating distributions rather than any particular model architecture. Just as the median is less sensitive to an adversary changing a minority of numbers than the mean, we give an analogous distributional aggregation that is less sensitive to adversarial component distributions compared to averaging (i.e., picking a random sample from a random distribution).

Rather than needing to certify that a particular model is safe, safety depends only on some fraction of a collection being safe, without knowing which ones. Such conditions are standard in fault-tolerant systems: replicated servers and redundant memory can remain reliable when some components fail, even when we do not know in advance which components will fail. Likewise, it may be more realistic to expect that some fraction of independently developed models are safe than to certify a particular model as safe.

Many approaches to AI safety rely on human, AI, or software inspection of model activations, outputs, or reasoning processes. However, some misalignment risks provably cannot be detected or mitigated through inspection alone, even with superintelligent AI oversight assistance \citep{anwar2024foundational}. \Cref{fig:steg-examples} illustrates the issue: real image outputs may look equally innocuous even when one was generated by a steganographic model that stealthily encodes an arbitrary secret message, and long code outputs may likewise hide a security vulnerability that is impossible to spot by inspection alone.
Therefore, any protocol resilient to such unsafe behavior must interact with the models beyond simply overseeing their responses.

Our aggregation scheme needs only black-box access to sample from each distribution and evaluate the probability of any candidate output. Generative models often provide exactly this access: autoregressive language models and multimodal models that natively generate images \citep{gpt4oaddendum} can compute $p(y \mid x)$ by multiplying successive token probabilities. In the steganographic example, inspecting the final image is not enough, but the steganographic source assigns much higher \textit{probability} to its own encoded sample than the ordinary sources do. This work extends techniques from \citet{vyas2023provable}, who used $\LM$ probabilities for provable copyright protection, applying them to a fundamentally different safety challenge.

\begin{figure}[tbp]
  \centering
  \begingroup
  \setlength{\tabcolsep}{0.25em}
  \begin{tabular}{cc}
  \includegraphics[width=0.71\textwidth]{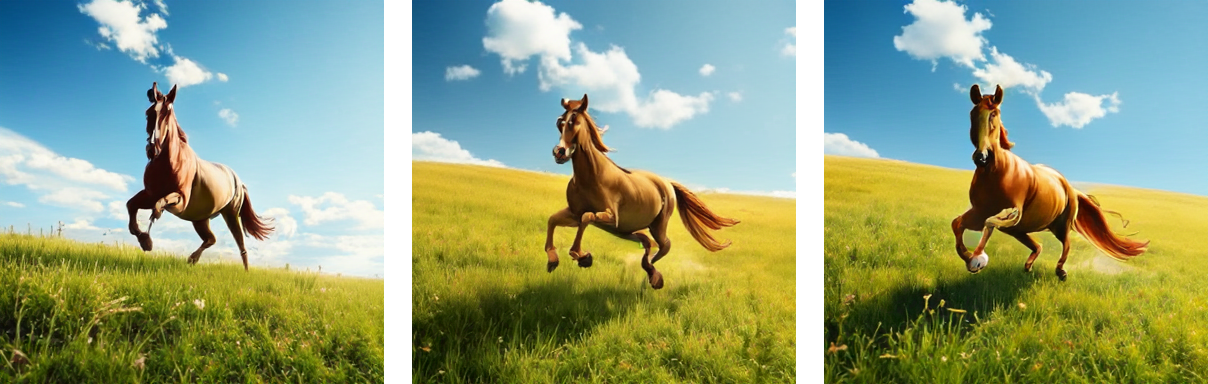}
  &
  \includegraphics[width=0.26\textwidth]{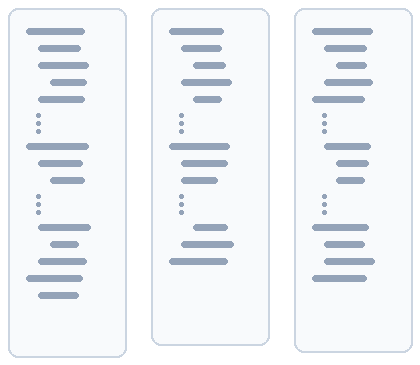}
  \end{tabular}
  \endgroup
  \caption{\textbf{Left:} three LLM outputs for the prompt ``a horse running through a meadow'', one of which secretly encodes a fictitious credit card number. From the image alone, such steganography is provably undetectable and provably doesn't alter quality \citep{witt2023perfectly,zamir2024undetectable}. However, generation probabilities reveal the suspicious output: the image on the right is much more likely according to the model that generated it than either of the other two. \textbf{Right:} three long code outputs, one of which contains a security vulnerability; such vulnerabilities can likewise be provably hard to detect from the output alone \citep{cohen87virus}.}
  \label{fig:steg-examples}
\end{figure}


Our key insight is that outputs that are generated with similarly high probabilities by multiple models are more likely to be safe, even without knowing which constituents are trustworthy. To build intuition, consider two independent samples $y_1 \sim p_1$ and $y_2 \sim p_2$, where one source may be untrusted. If it happened that $y_1 = y_2$, this surprising ``jinx'' coincidence provides natural justification for outputting $y_1$: even if one distribution is malicious, the other independently generated the same output.

Of course, for high-entropy responses, exact agreement is exceedingly rare. We thus generalize this idea using the generation probabilities. If $y_1 \sim p_1$ but $p_2(y_1) \approx  p_1(y_1)$, then outputting $y_1$ is justifiably safe: even if $p_1$ is adversarial, $p_2$ would have generated it with nearly the same probability.

Indeed, we argue that responses should be considered safe only if multiple models could have generated them. 
Rather than selecting the ``best'' or most confident response, our guarantees come from restricting to regions of agreement across the models. 
This may discard unique or seemingly superior outputs, yet such a cost is inherent, as a lone ``genius'' model could also be the malicious one. 
A cautionary example comes from cryptography: NIST standardized the pseudorandom generator Dual\_EC\_DRBG, that appeared as the most elegant and well-designed option at the time, only for it to be revealed years later that it contained an intentional NSA backdoor \citep{DualEC,Shumow2007DualEC}. 
Hence, we would allow the aggregation algorithm to \emph{abstain} from providing a response when it is unable to find such an agreement.

Our first observation is that a distribution proportional to the pointwise \emph{median} of the constituent probabilities assigned to each candidate output is resilient to a minority of unsafe constituents, provided that the trustworthy ones frequently agree on probabilities. Intuitively, even a small coalition of adversarial sources cannot substantially distort that median probability, and we formalize this intuition in subsequent sections.
To formalize this intuition, we need to: (1) Define and quantify the \emph{robustness} of such an aggregated distribution, and (2) Show we can efficiently sample from this distribution, or from a close approximation of it.

To this end, we present a simple consensus sampling algorithm that takes as input $k \ge 2$ distributions.
We provide it with the target number $1\leq s\leq k$ of safe distributions to which the guarantee is compared.
As in cryptography, it is easiest to think of $s$ of the $k$ distributions as safe and the remaining $a=k-s$ as arbitrary. They might be safe, unreliable, or even adversarial. The algorithm does not know which distributions are safe.
We also provide the algorithm with a \emph{risk tolerance} parameter $R$.
It either produces an output $y$ or \textit{abstains}; in prompted applications, the distributions are those induced by the models after receiving the prompt $x$.

We measure \emph{risk} as the probability of generating outputs from some set\footnote{Our results can trivially be extended to continuous risk functions $\E_{y \sim p}[h(y)]$ for a non-negative bounded function $h$, but we use binary unsafe sets $U$ for simplicity.} of unsafe outputs, though the set need not be known; for meaningful guarantees, the $s$ safe distributions must each assign low probability ($\ll 1/R$) to unsafe responses.

We prove that the presented algorithm's risk is at most $R$ times the average risk of the safest $s$ distributions. This holds for every fixed prompt $x$ and every unsafe set $U$, avoiding the challenge of comprehensively defining safety. This robustness guarantee is independent of any further assumptions on the distributions.
The probability of abstention, on the other hand, is affected by the parameters $R$ and $s$, and certain overlap properties of the safe distributions.
We provide a formal guarantee that if a majority of the distributions are safe then given a certain assumption on their overlap, the rate of abstention decays exponentially with~$R$.
In particular, the choice of the parameter~$R$ governs a tradeoff between risk and abstention rate, and we later show that this robustness--abstention tradeoff is Pareto optimal.

We show that the above property, which we call $R$-robustness, also implies a bound of $\log(R+1)$ bits on the information a response can leak, in the case of steganography.

The sampler returns outputs when safe distributions overlap and abstains when they do not; in generative-model applications, even safe models may have small overlap because of idiosyncratic training differences. \Cref{fig:visual} illustrates this basic distinction.

\paragraph{Misalignment risk.}
AI safety risks are often categorized as \textit{misuse} (humans causing harm with AI), \textit{misspecification/misunderstanding} (AI failing to capture human objectives), and \textit{misalignment} (AI pursuing fundamentally different goals despite ``understanding'' them). Misalignment risk 
is particularly well-suited for our approach: (a) more capable models will not naturally be better at avoiding such risk, as illustrated by the impossibility of detecting or removing steganography or malware through examination alone, (b) few existing safety techniques offer robust prevention hence the additional costs of consensus sampling may be worthwhile, and (c) a multiplicative factor increase in the risk compared to the safest models is tolerable, because safe models should extremely unlikely to pursue these misaligned goals. No human or AI, even a superintelligent overseer, can examine a program and determine whether it contains steganography or planted vulnerabilities, or remove it without altering its meaning \citep{zamir2024undetectable,perry2025robuststeganographylargelanguage,cohen87virus}. Thus, the probabilistic nature of consensus sampling may have applications to mitigating such risks.

\paragraph{Contributions.} The main contributions of this work are: (a) formulating a general robust aggregation problem for probability distributions, motivated by hard-to-detect failures in generative AI; (b) recognizing overlap as the quantity governing the robustness--abstention tradeoff; and (c) providing an efficient sampling algorithm with information-theoretic bounds that limit adversarial influence without relying on interpretability or a formal definition of safety.

\paragraph{Limitations and costs.}
We emphasize important limitations. First, the approach does not ensure safe outputs: if all input models produce unsafe outputs for a prompt, the procedure will too. It requires multiple safe models.\footnote{For $s=1$, safety bounds still hold, but an adversary can force abstention.} It does not provide methods for training safe models; rather, it explores how to aggregate existing models and must be integrated with other methodologies for training safe AI systems. The bounds apply per-prompt; repeated calls can accumulate information leakage. A crucial limitation is that even benign language models may have very little overlap, as shown in experiments by \citet{vyas2023provable}. Thus, we discuss ways to employ the technique on certain pipelines to increase overlap, but this is mostly left as an interesting line of future inquiry. The set-based notion of safe and unsafe outputs does not encompass societal harms or harms arising from multiple interactions, both of which are crucial to avoid. Most fundamentally, this work represents one narrow investigation into model-agnostic properties, not a complete safety solution. 

\paragraph{Organization.}
\Cref{sec:notation,sec:consensus,sec:robust,sec:abstain} present the formal setup and main guarantees; \Cref{sec:experiments} gives synthetic and image-generation experiments; \Cref{sec:related} discusses related work; and \Cref{sec:conclusions} concludes. The appendices contain a median-distribution proof, overlap discussion, optimality, and leakage extensions.

\section{Notation and Oracle Access Assumption}\label{sec:notation}
Fix natural numbers $s\le k$, where $s$ is the target number of safe distributions and $a\defeq k-s$ is the number of arbitrary ones. We work directly with distributions $p_i$ over a finite or countable output set $Y$; in generative-model applications, these are the distributions induced by the models for a fixed prompt. All logs are base~2. For $p\in\Distr{Y}$ and $U\subseteq Y$, write $p(U)\defeq\sum_{y\in U}p(y)$. For $\pp=(p_1,\ldots,p_k)\in\Distr{Y}^k$, let $p_{(i)}(y)$ and $p_{(i)}(U)$ denote the $i$th smallest pointwise probability and set probability, respectively:
\begin{equation}\label{eq:order}
p_{(1)}(y) \le p_{(2)}(y) \le \ldots \le p_{(k)}(y) ~\text{ and }~ p_{(1)}(U) \le p_{(2)}(U) \le \ldots\le \ldots p_{(k)}(U).
\end{equation}
Let $[k]\defeq\{1,\ldots,k\}$, and let $\bot\notin Y$ denote abstention.


\paragraph{Distribution access assumption.}
We assume oracle access to sample from each $p_i$ and evaluate $p_i(y)$ for any candidate $y$. Autoregressive language models support these operations by sampling and multiplying next-token probabilities; \Cref{sec:prompts} states the prompted version formally.

\section{Warm-up: the median distribution}\label{sec:med-main}
Before giving the efficient algorithm, it is helpful to consider the pointwise median
$$p_{\med}(y) \propto \text{median}\bigl(p_1(y), p_2(y), \ldots, p_k(y)\bigr).$$
If a majority $S$ of the distributions are safe and overlap by
$$\Delta(S)\defeq \sum_y \min_{i\in S} p_i(y),$$
then \Cref{thm:med} in the appendix shows that
$$p_{\med}(U)\le \frac{1}{\Delta(S)}\sum_{i\in S}p_i(U)$$
for every unsafe set $U$. Thus overlap converts agreement among safe distributions into a direct risk bound. The median is only a warm-up: it is not known how to sample from it efficiently, it requires overlap knowledge about the unknown safe set, and it needs a safe majority. Consensus sampling keeps the same underlying idea while providing an efficient abstaining algorithm whose robustness guarantee does not require knowing the safe set or its overlap.

\section{Consensus Sampling}\label{sec:consensus}
In the consensus sampling problem, we are given multiple distributions from which we can sample and compute probabilities, and we must produce an output or $\bot$. \Cref{alg:efficient} takes a lower bound~$s$ on the number of safe distributions and a risk-tolerance parameter~$R\geq 0$. Each round draws from the uniform mixture, sorts the $k$ probabilities assigned to the candidate, and accepts with the ratio of the average of the $s$ smallest probabilities to the average over all $k$ probabilities. The ratio is at most~1, and if no candidate is accepted after $R$ rounds the algorithm abstains.

There is a tradeoff between abstention and safety, as the algorithm which always abstains is considered safe by definition. A smaller $s$ would yield safer outputs, where one can view $s=0$ as always abstaining and $s=k$ as never abstaining and simply outputting a random sample from a random $p_i$.
Similarly, the parameter $R$ represents a risk tolerance, with smaller $R$ yielding safer outputs but higher abstention rate: $R=0$ would correspond to always abstaining and $R=\infty$ to never abstaining. \Cref{sec:prompts} analyzes the straightforward extension to include the prompt as an input. 

\begin{algorithm}[tb]
  \caption{Consensus sampling from $k$ distributions, risk-competitive with the safest $s$}
  \label{alg:efficient}
  \begin{algorithmic}[1]
    \Require Integers $k \ge s \ge 1$, $R \ge 0$, and distributions $p_1,p_2,\ldots, p_k \in \Distr{Y}$
    \Ensure A sample $y\in Y$ or $\bot$
    \For{$r = 1$ to $R$}
      \State \label{step:sample} Sample $y \sim \frac{1}{k}\sum_{i=1}^k p_i$ \Comment{pick a random index $i$ and sample from $p_i$}
      \State \label{step:acceptance}With probability $\dfrac{\frac{1}{s}\sum_{i\le s} p_{(i)}(y)}{\frac{1}{k}\sum_{i=1}^k p_i(y)}$ \textbf{return} $y$  
    \EndFor
    \State \Return $\bot$
  \end{algorithmic}
\end{algorithm}

We will analyze its efficiency, safety, and abstention rate. For efficiency, the algorithm is pseudopolynomial time in $R$:
\begin{lemma}[Efficiency]\label{lem:efficiency}
  For any constant $k$, 
  \Cref{alg:efficient} can be computed using $O(R)$ oracle calls and additional arithmetic operations. 
\end{lemma}
\begin{proof}
  To draw a sample $y$ in line \ref{step:sample}, choose $i \in [k]$ uniformly at random and then draw a single sample $y \sim p_i$. To compute the probability ratio (line \ref{step:acceptance}) 
  then requires $k$ evaluations of probabilities $p_i(y)$, one for each index $i \in [k]$. A constant number of additional operations per step are required, since $k$ is taken to be a constant. 
\end{proof}

To better understand and analyze \Cref{alg:efficient}, we note that it uses rejection sampling to sample from the following distribution $\jinx$, stopping after at most $R$ iterations.
\begin{equation}\label{eq:jinx}
    \jinx(y) \defeq \frac{1}{Z} \cdot\frac{1}{s}\sum_{i=1}^s p_{(i)}(y) ~~\text{ for }~~ Z \defeq \sum_y\frac{1}{s} \sum_{i=1}^s p_{(i)}(y).    
\end{equation}
\begin{lemma}\label{lem:exact}
For any $\pp=(p_1,p_2,\ldots,p_k)\in \Distr{Y}^k,$ $R \ge 0$, and for $\jinx, Z$ from \Cref{eq:jinx}, the output distribution $q$ of \Cref{alg:efficient} satisfies:
\begin{equation}\label{eq:q}
q(\bot) = (1-Z)^R ~~\text{ and }~~q(y) = \bigl(1-q(\bot)\bigr)\, \jinx(y) \text{ for all } y \in Y.
\end{equation}
\end{lemma}
Thus, ordinary rejection sampling (essentially $R=\infty$ if $Z>0$) gives $q(\bot)=0$  and an output distribution of $\jinx$.
\begin{proof}
The lemma trivially holds for $R=0$, because $q(\bot)=1$. For shorthand, write
  \begin{equation}\label{eq:defs}
    f(y) \defeq \frac{1}{k}\sum_{i=1}^k p_i(y),
    \qquad
    g(y) \defeq \frac{1}{s}\sum_{i=1}^s p_{(i)}(y), \qquad   \alpha(y) \defeq \frac{g(y)}{f(y)}
  \end{equation}
  so $y\sim f$ in line \ref{step:sample} and the acceptance probability in line \ref{step:acceptance} is $\alpha(y)$. 
  Fix $y\in Y$. For $R=1$, the probability of outputting $y$ is
  $$f(y)\cdot \alpha(y) = g(y)=Z \jinx(y) = (1-q(\bot))\jinx(y)$$
  and hence $q(\bot)=1-Z$, which establishes the lemma for $R=1$. For $R>1$, note that the distribution of outputs is the same each round, conditional on making it to that round. Thus, $q\propto q^*$ and the constant of proportionality is of course $1-q(\bot)$ because $q(Y)=1-q(\bot)$ while $q^*(Y)=1$.
\end{proof}

We note that the natural generalization of the median distribution to arbitrary $s$, the distribution that is pointwise proportional to $p_{(s)}(y)$, could be used in place of $\sum_{i \le s} p_{(i)}(y)$ in \Cref{alg:efficient} with the same bounds in the paper. However, we will show that $\jinx$ and $\Cref{alg:efficient}$ are optimal in terms of safety and abstention properties.

\section{Robustness properties}\label{sec:robust}

In this section we discuss robustness properties relating to a set of unsafe outputs $U \subseteq Y$, with $y \notin U$ considered safe. When $U$ is clear from context, we refer to the probability that a generation from a distribution $p$ is unsafe, $p(U)$, as the \textit{risk} of $p$. This general set-based view of safe outputs enables us to analyze \Cref{alg:efficient} without having to pick a single definition of safety, such as malicious introduction of code. We first introduce a property called $R$-robustness where the parameter $R \ge 0$ captures the excess multiplicative risk incurred by using a distribution $q$, which may abstain, compared to the ``safest'' $s$ distributions (those with smallest risk). Intuitively, we compare the risk $q(U)$ to the $s$ safest distributions. We will then prove that \Cref{alg:efficient} satisfies this definition and relate robustness to steganography.

Recall that the order statistic $p_{(s)}(U)$ denotes the $s$th smallest of the risks $p_1(U),\ldots, p_k(U)$, as defined by \Cref{eq:order}. The following definition of $R$-robustness requires being competitive with $p_{(s)}(U)$ for \textit{all possible unsafe sets} $U$. It is slightly stronger than the simpler $q(U) \le R \cdot p_{(s)}(U)$ requirement discussed in the introduction, because $p_{(s)}(U) \ge \sum_{i\le s} p_{(i)}(U)/s$:
\begin{definition}[Consensus robustness]\label{def:robust}
    For $R \ge 0$, $q\in \Distr{Y \cup \{\bot\}}$ is consensus-robust with parameter $R$, or just 
    \textit{$R$-robust}, relative to $\pp=(p_1,p_2,\ldots, p_k) \in \Distr{Y}^k$ if
    $$q(U) \le R \cdot\frac{1}{s} \sum_{i=1}^{s} p_{(i)}(U) \text{ for all } U \subseteq Y.$$
\end{definition}
Smaller $R$ indicates less risk. Because abstaining is viewed as safe, the distribution that always abstains is 0-robust. We use the same variable $R$ in this definition because our upper bound on $R$ exactly matches the input $R$ to \Cref{alg:efficient}. 

The following theorem states that the risk of \Cref{alg:efficient} is at most the average risk of the safest $s$ distributions.
\begin{theorem}[Consensus robustness]\label{thm:robustness}
For any $R \ge 0$ and $\pp=(p_1,p_2,\ldots, p_k) \in \Distr{Y}^k$,
the output distribution of Algorithm~\ref{alg:efficient} is $R$-robust relative to $\pp$. 
\end{theorem}
\begin{proof}
Let $j_1, j_2, \ldots, j_s \in [k]$ be distinct integers so that $p_{(i)}(U) = p_{j_i}(U)$ for each $i \le s$, i.e., the indices of the least likely $s$ distributions. For any $y \in Y$, by \Cref{lem:exact}, since $1-(1-Z)^R \le RZ$ (from the classic union bound), 
$$q(U) \le  RZ q^*(U) = \frac{R}{s}\sum_{y \in U}\sum_{i=1}^s p_{(i)}(y) \le \frac{R}{s}\sum_{y \in U}\sum_{i=1}^s p_{j_i}(y) = \frac{R}{s}\sum_{i=1}^s p_{j_i}(U).$$
By assumption on $j_i$, this last quantity is equal to $R \cdot \frac{1}{s}\sum_{i=1}^s p_{(i)}(U)$, as needed.
\end{proof}

We note that \Cref{thm:robustness} has immediate adversarial implications when there is a set $S \subseteq [k]$ of low-risk ``safe'' distributions, while the other distributions can be adversarially controlled.
\begin{corollary}[Adversarial robustness]
Let $U \subseteq Y$, $\pp \in \Distr{Y}^k$, $S \subseteq [k]$ of size $|S|=s$, with mean risk $\eps\defeq \sum_{i \in S} p_i(U)/s$. Then 
\Cref{alg:efficient}'s output distribution has risk $\le \eps R$.
\end{corollary}
\begin{proof}
    The proof follows simply from $\frac{1}{s}\sum_{i \le s} p_{(i)}(U) \le \frac{1}{s}\sum_{i \in S} p_i(U)$.
\end{proof}

\paragraph{Backdoors and jailbreaks.}
\Cref{alg:efficient} may be more useful at thwarting backdoors than jailbreaks in current $\LM$'s. To see why, note that as long as $s$ of the models produce safe $y \notin U$ on a certain ``backdoor'' prompt $x$, then \Cref{alg:efficient} will also produce a safe output. Thus, an adversary using techniques such as \citep{kurita-etal-2020-weight,hubinger2024sleeperagentstrainingdeceptive} to plant a backdoor would have to plant the same backdoor $x$ in $>a=k-s$ of the models. For jailbreaks, though, current $\LM$s are often susceptible to ``universal'' jailbreaks \citep{wallace-etal-2019-universal} which work across many models, regardless of training platform. Therefore, for \Cref{alg:efficient} to be useful in thwarting jailbreaks would require increased model safety from further safety training innovations.

In \Cref{sec:minmax}, we show that $\jinx$ is $1/Z$-robust and in fact is optimal, being the only distribution that is $1/Z$-robust among all distributions that never abstain. This in turn, implies that our algorithm is Pareto optimal in terms of its robustness-abstention tradeoff.

\subsection{Steganography and leakage}\label{sec:leakage}

This section argues that $R$-robustness implies resilience to steganography and that 
\Cref{alg:efficient} ``leaks'' few bits. There are many possible ways to quantify leakage; we consider three, all bounded by $\log(R+1)$ bits. The first uses simulation, a common cryptographic approach. In particular, we show that any choice of the unsafe models can equivalently be simulated by a procedure whose only dependence on the unsafe models is through a single integer $1 \le t \le R+1.$ Since such an integer can be written down using $\log(R+1)$ bits (technically $\lceil \log(R+1)\rceil$), this can be viewed as  $\le \log(R+1)$ bits of leakage.
\begin{lemma}\label{lem:sim}[\Cref{alg:efficient} leakage simulation bound]
Fix any $R \ge 0$ and size-$s$ $S\subseteq [k]$. There exists a randomized function $\mathcal{A}: Y^R \times \Distr{Y}^k \rightarrow [R+1]$ such that the following procedure has the same output distribution as \Cref{alg:efficient}, for any $\pp=(p_1,\ldots,p_k) \in \Distr{Y}^k$:
\begin{algorithmic}[1]
  \State Sample $y_1,\ldots,y_R \stackrel{\text{i.i.d.}}{\sim} \frac{1}{s}\sum_{i\in S} p_i$
  \State $t \gets \mathcal{A}\big(y_1,\ldots,y_R,\,\pp\big)$
  \State \textbf{if} $t=R{+}1$ \textbf{ then} \Return $\bot$ \textbf{else} \Return $y_t$
\end{algorithmic}
\end{lemma}
The point is that only dependence on the unsafe $\langle p_j\rangle_{j \notin S}$ is via $t$ in the above procedure, whose output distribution is equivalent to \Cref{alg:efficient}. The $R$ independent samples are computed solely from the uniform mixture over safe distributions. Thus, an adversary's control over the output can be boiled down to a single number $t \in [R+1]$, describable in $\log(R+1)$ bits. 
\begin{proof}
    The algorithm $\mathcal{A}$ iterates through $t=1,2,\ldots, R$ and accepts $y_t$ with probability,
    $$\dfrac{\sum_{i\le s} p_{(i)}(y_t)}{\sum_{i\in S} p_i(y_t)} \le 1.$$  
    This is at most 1 because the numerator is summed over the smallest $s$ probabilities and the denominator is only over the $s$ safe ones. The per-round probability of outputting any $y \in Y$ is easily seen to be $\sum_{i\le s} p_{(i)}(y)/s$ exactly as shown for \Cref{alg:efficient} in \Cref{lem:exact}.
\end{proof}

We next argue that $R$-robustness, in general, implies resilience to steganography, as discussed briefly in \Cref{sec:med}. First we bound risk. Again, suppose there is a decoding function $f: Y \rightarrow \{0,1\}^*$ used to recover a secret message $m \in \{0,1\}^*$ encoded as, say, a binary string. Say the safe distributions have some low (presumably exponentially small) $\eps$ chance of generating a sample  $y \sim p$ that happens to decode to the secret $m=f(y)$. The following shows that $R$-robustness implies $\le \eps R$ risk due to steganography.
\begin{observation}\label{obs:steg}
Let $R \ge 0$ and $q$ be $R$-robust relative to $p_1,p_2,\ldots,p_k$. Fix $m \in \{0,1\}^*$ and a decoding function $f: Y \rightarrow \{0,1\}^*$. Let $\eps \in [0,1]$ and $S$ be a set of size $s$ such that $\Pr_{y \sim p_i}[f(y)=m] \le \eps$ for all $i \in S$. Then, 
$$\Pr_{y \sim q}[f(y)=m] \le \eps R.$$
\end{observation}
This holds for \Cref{alg:efficient}, since \Cref{thm:robustness} shows that its output distribution is $R$-robust.
\begin{proof}
Let $U=f^{-1}(m)$, so $p_i(U) = \Pr_{y \sim p_i}[f(y)=m] \le \eps$ for all $i \in S$. By \Cref{def:robust},
$$q(U) \le R\cdot \frac{1}{s} \sum_{i=1}^s p_{(i)}(U) \le R \cdot \frac{1}{s}\sum_{i \in S} p_i(U) \le R \cdot \eps.\qedhere$$
\end{proof}

This shows that the success rate of a steganographic attack is low, but we can go further and provide an upper bound on the number of bits leaked by our algorithm using \textit{mutual information}, commonly used to quantify average information leakage between two random variables. It requires additional notation and definitions to apply to our setting. \Cref{sec:mutinfo} proves that any $R$-robust distribution $q$ leaks at most $O(\log R)$ bits in the sense that the mutual information between a sample from $q$ and a random secret message to be leaked is $\le \log(R+1)$. It is extended to the case of prompts 
For a small image-generation illustration of this setting, see \Cref{sec:imageexp}.

\section{Abstention bounds}\label{sec:abstain}

\Cref{lem:exact} shows that the algorithm's abstention rate is exactly $q(\bot)=(1-Z)^R$ but the quantity $Z \defeq \sum_y\sum_{i\le s} p_{(i)}(y)/s$, could be minuscule or zero (always abstaining is safe but useless). We now lower-bound $Z$ even in the presence of an adversary that controls a minority $< k/2$ of the distributions. It relies on overlap, which recall is $\Delta(S)\defeq \sum_{y \in Y} \min_{i \in S} p_i(y)$.

We will be able to lower-bound the overlap with a safe majority, i.e., $s>k/2$. For instance, for $k=3$ distributions, if the $s=2$ safe distributions overlap by $\Delta(S) \ge 0.2$, then the analysis below shows that the algorithm abstains with exponentially low probability $\le 0.9^R$.
However, if the distributions rarely agree on likely outcomes, then our aggregation protocol cannot hope to robustly output a common sample with significant probability. 

\begin{figure}[tbp]
\centering
    \includegraphics[width=0.7\textwidth]{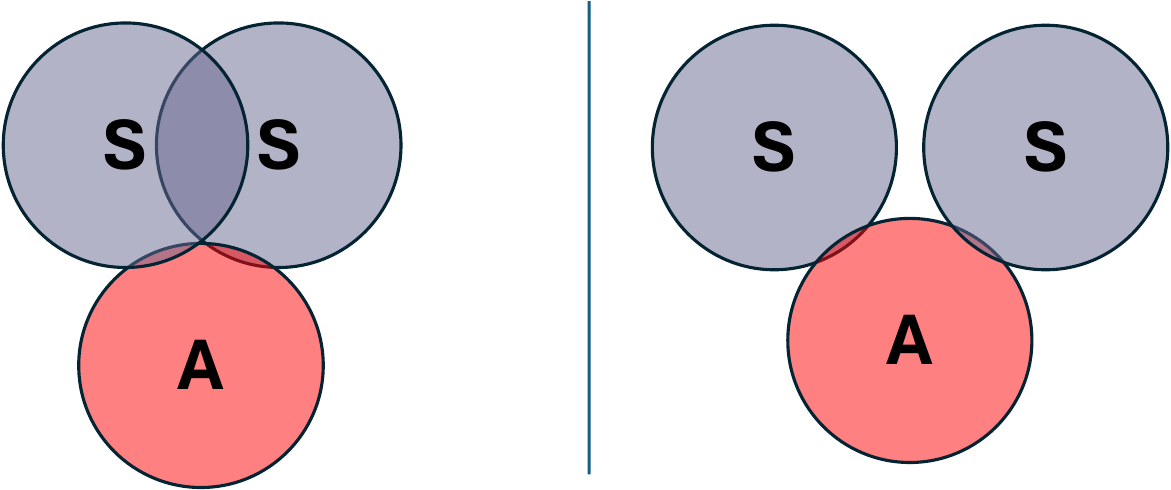}
    \caption{Suppose an adversarial model has a distribution uniform over unsafe responses, shown in red, while safe distributions are shown in silver. Left: with sufficient overlap among safe distributions, consensus sampling returns a point from the overlap region, which is mostly safe. Right: with no overlap between safe distributions, the algorithm abstains.}
    \label{fig:visual}
\end{figure}

While $Z \ge \Delta(S)/s$ follows from  
\Cref{thm:abstention} below, the theorem gives a tighter bound $Z \ge \Delta_a(S)/s$ where the \emph{maximal overlap} $\Delta_a(S) \ge \Delta(S)$, defined as follows:
\begin{definition}[$c$-overlap]\label{def:coverlap}
For an index set $I \subseteq [k]$ and $c \in \{0, 1, 2, \ldots, |I|-1\}$:
$$\Delta_c(I) \defeq \max_{\substack{J \subseteq I\\|J|=c+1}} \Delta(J).$$
\end{definition}
It is easy to see that $0 \le \Delta(I) \le \Delta_c(I) \le 1$ because $J \subseteq I \Rightarrow \Delta(J) \ge \Delta(I)$.
\begin{theorem}[Abstention]\label{thm:abstention}
Assume $s>k/2$ and fix any $R\ge 1$ and $\pp \in \Distr{Y}^k$. For the output $y$ of \Cref{alg:efficient},
\[
\Pr[y=\bot]  \le \min_{S\subseteq [k]: |S|=s}  \left(1-\frac{\Delta_a(S)}{s}\right)^R.
\]
\end{theorem}
Recall that $a=k-s$. The safe-majority condition $s>k/2$ ensures that any set of $a+1=k-s+1$ indices must contain at least one of the $s$ smallest probabilities, which is the crux of the bound.
\begin{proof}
  It suffices to show the bound for any fixed size-$s$ $S$.
  Fix $S^\star\subseteq S$ of size $|S^\star|=a+1$ that attains $\Delta(S^\star)=\Delta_a(S)$. Next, we claim that $p_{(s)}(y) \ge \min_{i\in S^\star}p_i(y)$
  for any $y \in Y$ because $p_{(s)}(y)$ is the $s$th smallest probability, and $S^\star$ has all but $s-1$ elements of $[k]$ and thus must contain one index $i$ such that $p_i(y) \le p_{(s)}(y)$. (This is where $s > k/2$ is used.)

By \Cref{lem:exact}, $\Pr[y=\bot] = (1-Z)^R$ for
$$Z=\sum_{y \in Y} \frac{1}{s}\sum_{i\le s} p_{(i)}(y) \ge \sum_{y \in Y}\frac{1}{s} p_{(s)}(y)\ge \frac{1}{s}\sum_y \min_{i\in S^\star}p_i(y)= \frac{\Delta(S^\star)}{s}= \frac{\Delta_a(S)}{s}.$$
  \end{proof}

For $a=1$, the bound is $Z \ge \Delta_1(S)/s$ where $\Delta_1(S)$ is the maximum overlap between any two safe distributions.  Unfortunately, as discussed before, this quantity may be exponentially small in the output length for language models that were not designed with overlap in mind. \Cref{sec:engoverlap} offers suggestions for ways to increase overlap and reduce abstention.
At the same time, overlap is more natural in constrained response regimes than in free-form long generation: for binary decisions, agreement and overlap coincide exactly, and tightly specified outputs such as chemical formulas or factual identifications may admit substantial overlap even when the answer space is unknown.

\section{Experiments}\label{sec:experiments}
We use two small experiments for different purposes. The first isolates the probability-aggregation mechanism from language models; the second instantiates the motivating steganography example with an open image-token model.

\subsection{Synthetic hypercube experiment}\label{sec:hypercubeexp}
To emphasize that consensus sampling is a general probability-aggregation method rather than a construction specific to language models, we ran a controlled synthetic experiment on the hypercube \(Y=\{0,1\}^{64}\). We used \(k=3\) sources and \(s=2\) safe ones. The first safe source was uniform on the hypercube, while the second sampled each bit independently with probability \(1/2+\epsilon\) of being one, with \(\epsilon=0.02\). The unsafe source was uniform except that a keyed random subset of \(m\) coordinates was fixed to an \(m\)-bit secret. The unsafe set consisted of strings whose keyed coordinates decoded to that secret. Thus the uniform safe source hit the unsafe set with probability exactly \(2^{-m}\), while the biased safe source had a nearby probability depending on the secret bits.

It may seem surprising that the unsafe probability is so large for \(m=1\), but with only one bit of steganography, a \textit{random} sample already encodes the secret with probability \(1/2\). More generally, monkeys at typewriters have probability \(2^{-m}\) of generating any given \(m\)-bit secret by chance. This makes the sweep useful pedagogically: when only one or two bits are hidden, the unsafe event is common even under the safe sources, so consensus sampling has little reason to suppress it; as \(m\) grows and the safe-source probability of the secret becomes small, unsafe outputs become an increasingly distinctive signature of the adversarial source.

We swept the hidden message length over \(m\in\{1,2,4,8,16\}\) and the risk tolerance over \(R\in\{1,\ldots,5\}\), using 20,000 Monte Carlo trials per setting. For each plotted rate in \Cref{fig:hypercube-steg-sweep}, we treat the trials as Bernoulli observations and report the two-sided 95\% Wilson confidence interval for the underlying probability. The resulting curves show the robustness--abstention tradeoff predicted by the theory: longer hidden messages are rarer under the safe sources, so consensus sampling suppresses them more strongly, while larger \(R\) lowers abstention. This synthetic experiment ran on a laptop CPU in under one minute.

\begin{figure}[tbp]
\centering
\includegraphics[width=0.84\textwidth]{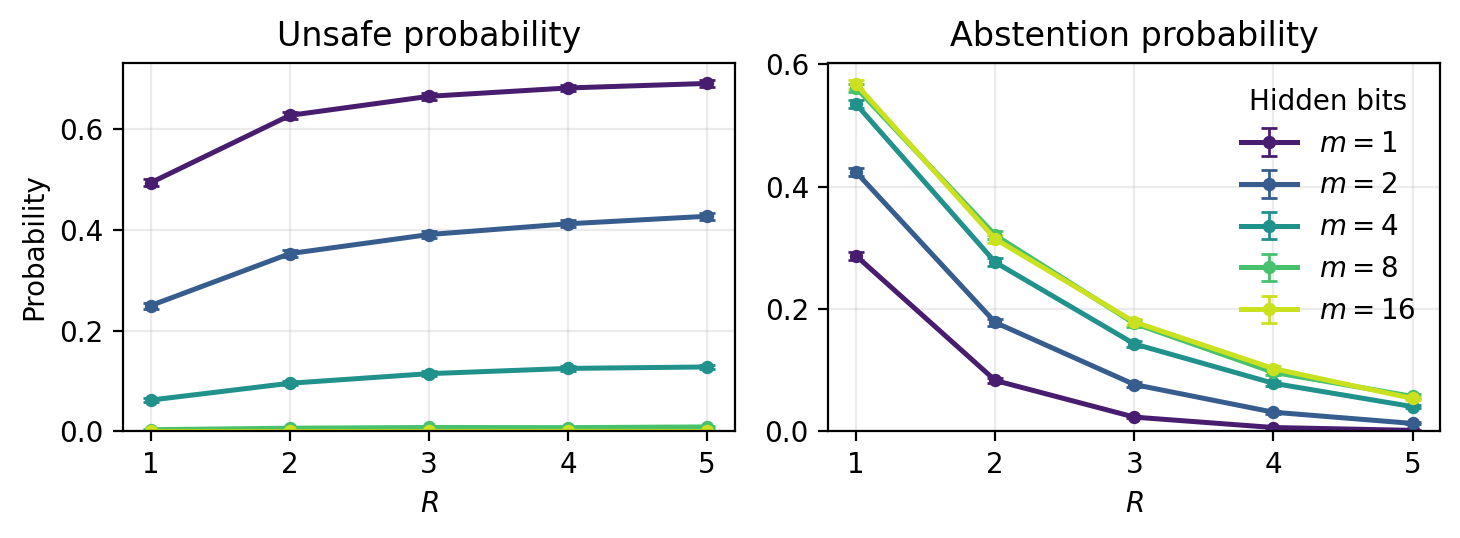}
\caption{Synthetic hypercube experiment with one uniform safe source, one slightly biased safe source, and one keyed steganographic source. Left: unsafe-output probability under consensus sampling. Right: abstention probability. Error bars, barely visible at this scale, show two-sided 95\% Wilson confidence intervals over 20,000 trials per setting.}
\label{fig:hypercube-steg-sweep}
\end{figure}

\subsection{Image steganography experiment}\label{sec:imageexp}
We next return to the steganography application of \Cref{fig:steg-examples} with Janus-Pro-7B image tokens and the public scheme of \citet{zamir2024undetectable}. In the idealized random-function view of the construction, averaging over the secret randomness leaves the base output distribution exactly unchanged; with a pseudorandom key, the formal guarantee is computational indistinguishability. We used Janus-Pro-7B to generate \(384\times384\) images directly from 30 short image prompts such as ``a horse running through a meadow'' or ``a lighthouse during a storm.'' The 30 prompts were drafted with assistance from OpenAI Codex and then manually selected and edited. For each prompt, we sampled a fictitious 20-digit credit card number that was not shown to the image model and served only as the hidden message.

The two ordinary sources, \(p_1\) and \(p_2\), were independent ordinary samples from the Janus distribution conditioned on the image prompt. The third source, \(p_3\), used the public implementation accompanying \citet{zamir2024undetectable}, adapted from text tokens to Janus image tokens, to encode the binary representation of the hidden credit card number while generating from the same image prompt. We sampled five images from each source for each prompt. The steganographic decoder recovered the full 67-bit number from all 150 steganographic images.

For this deliberately simple illustration, the adversarial source was keyed and deterministic: for a fixed prompt, sample index, and key, it assigned probability one to its own steganographic trajectory and zero to off-path trajectories. Thus, the experiment is intended as a sanity check that the consensus-sampling mechanism behaves as predicted on real image-token distributions, rather than as a benchmark for optimized steganographic attacks. Under the ordinary mean mixture, a uniformly drawn source is unsafe with probability \(1/3\) by construction and never abstains. We then ran \Cref{alg:efficient} with \(k=3\) and \(s=2\) on 30 prompt instances. As shown in \Cref{tab:image-steg-results}, consensus sampling never returned a steganographic image in this run, while its abstention rate fell to zero by \(R=4\). Representative samples appear in \Cref{fig:image-steg-examples}; the unsafe source is not visually apparent from the generated image alone.

\begin{table}[tbp]
\centering
\small
\begin{tabular}{lcc}
\toprule
Method & Abstention rate & Unsafe rate \\
\midrule
Uniform mean mixture & \(0.00~[0.00, 0.12]\) & \(0.33~[0.17, 0.53]\) \\
Consensus sampling, \(R=1\) & \(0.37~[0.20, 0.56]\) & \(0.00~[0.00, 0.12]\) \\
Consensus sampling, \(R=2\) & \(0.23~[0.10, 0.42]\) & \(0.00~[0.00, 0.12]\) \\
Consensus sampling, \(R=3\) & \(0.10~[0.02, 0.27]\) & \(0.00~[0.00, 0.12]\) \\
Consensus sampling, \(R=4\) & \(0.00~[0.00, 0.12]\) & \(0.00~[0.00, 0.12]\) \\
\bottomrule
\end{tabular}
\caption{Image steganography sanity check. Brackets show two-sided 95\% exact binomial confidence intervals over 30 prompt instances. The unsafe point rate for the mean mixture is \(1/3\) by construction because exactly one of the three sources is steganographic; for comparability, its interval is the exact-binomial interval corresponding to 10 unsafe outcomes among 30 trials.}
\label{tab:image-steg-results}
\end{table}

\begin{figure}[tbp]
\centering
\includegraphics[width=0.78\textwidth]{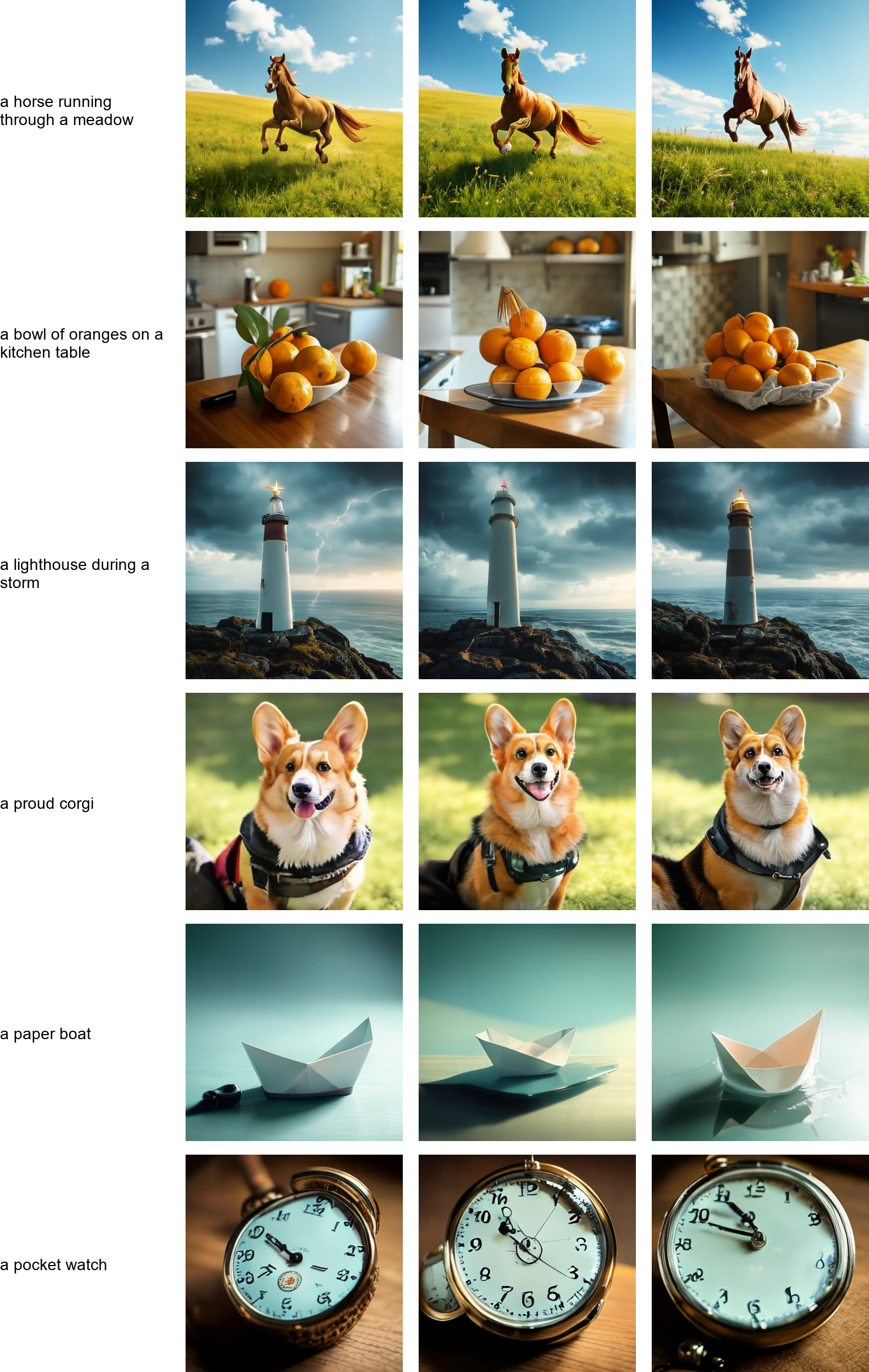}
\caption{Representative image outputs from six prompts. Each row contains two ordinary images and one steganographic image in shuffled order; the consensus sampler is not told which source is unsafe. The steganographic images are visually ordinary despite encoding the hidden number.}
\label{fig:image-steg-examples}
\end{figure}

The Janus code repository is MIT-licensed, the Janus-Pro-7B weights are released under the DeepSeek Model License, and the steganography implementation used in our supplementary code is released with permission from its author. The reported run used four NVIDIA GB300 GPUs for about 10 minutes wall-clock, about 35 GPU-minutes in total.

\section{Related Work}\label{sec:related}
This paper is related to prior work in information-theoretic safety, ensemble methods, hardness results for detection problems, and broader AI safety approaches.

\paragraph{Information-theoretic foundations.} 
Our approach is inspired by the elegant work of \citet{vyas2023provable} which uses ensemble methods for provable copyright protection by training separate generative models on disjoint data partitions and sampling only from regions of agreement. Their CP-$\Delta$ algorithm is analogous to the special case $s=1$ in our \Cref{alg:efficient}, and our consensus-robustness notion parallels their $k$-Near Access-Free definition with Renyi divergence of order infinity. However, their focus is the non-adversarial problem of copyright protection through training data partitioning, while we address a broader class of safety risks, including misalignment risk such as steganography and code vulnerabilities. Differential privacy \citep{dwork2006calibrating} provides another information-theoretic framework for safety, establishing bounded dependence on training data for privacy-preserving machine learning. It has been incorporated into generative models \citep{Li2022LLMsDP}. Differential privacy is similar in spirit but technically different because, in our setting, a single adversarial model can drastically shift the distribution. 

\paragraph{Ensemble methods and probability aggregation.}
The use of multiple models has a rich history in machine learning. Classical ensemble methods combine predictions to boost accuracy and reduce variance \citep{dietterich2000ensemble}. In the context of language models, several recent works leverage ensembles specifically for safety and reliability objectives. Multiple prior works combine language model probabilities to improve safety and reliability, though with different goals and guarantees than our approach. DExperts \citep{Liu2021} uses a combination-of-experts at decoding time to reduce toxic content, while self-consistency methods \citep{wang2022selfconsistency} sample multiple reasoning paths and select the most consistent answer to reduce logical errors. 
These methods illustrate the power of using multiple models or prompts in tandem to cross-check and filter generations for safety and fidelity. However, this prior work does not address risks such as steganography or code vulnerabilities, which are difficult to detect and may be particularly concerning for superintelligent AI.

\paragraph{Fundamental hardness results.}
One motivation for our work comes from computational impossibility results that limit inspection-based safety approaches. \citet{cohen87virus} showed that detecting security vulnerabilities such as viruses is computationally intractable. More recently, \citet{zamir2024undetectable,witt2023perfectly} proved that steganography detection is similarly intractable, demonstrating an efficient language model scheme for encoding arbitrary messages indistinguishably from ordinary outputs. \citet{perry2025robuststeganographylargelanguage,motwani2024secret}  show that hidden messages can be encoded such that removing them necessarily alters semantic meaning, even with unlimited computational resources.

These hardness results have profound implications: no human or AI system, regardless of capability, can reliably detect or remove such embedded information through output inspection alone. This suggests that inspection-based safety approaches may have certain limitations and motivates using output probabilities. In practice, while methods like paraphrasing \citep{Krishna2023Paraphrasing} can often destroy hidden messages, advanced steganographic techniques have emerged which survive such interventions \citep{mathew2024hiddenplaintextemergence}. Our approach circumvents these hardness results by aggregating probabilities rather than inspecting outputs.

\paragraph{Broader AI safety landscape.}
Our work complements but differs fundamentally from mainstream AI safety approaches \citep[see, e.g.,][]{anwar2024foundational}. Supervised oversight methods \citep{bowman2022measuring} typically involve one model evaluating another's outputs, but the hardness results above suggest this may be insufficient for provably undetectable risks. AI debate frameworks \citep{irving2018ai} decompose complex claims into verifiable atomic steps, but similarly it is not clear how a claim like ``this image contains no secret messages'' would fit into the debate framework. Mechanistic interpretability \citep{Bereska2024MechInterpReview} aims to understand model internals but generally lacks formal guarantees that the interpretation is correct. Nonetheless, we emphasize that our method provides no training methodology for safe models. It must therefore be integrated with other safety training approaches rather than replacing them.

\paragraph{Fault-tolerant computing and cryptographic perspectives.}
Our framework draws inspiration from fault-tolerant distributed systems, where consensus protocols achieve robustness despite adversarial participants. The Byzantine generals problem \citep{Lamport1982} establishes that systems can tolerate up to one-third adversarial participants. However, our consensus sampling problem is more forgiving: we need not achieve agreement on the models' distributions themselves, but merely draw samples likely to be safe given distributional overlap. Recent work applies cryptographic insights to AI safety \citep{Goldwasser2024Backdoors,ball2025impossibilityseparatingintelligencejudgment,zamir2024undetectable}, though many such results are negative impossibility findings. Our contribution differs by providing a constructive algorithm with positive information-theoretic guarantees.

\section{Conclusions, Limitations, and Future Work}\label{sec:conclusions}

We introduce a black-box framework for aggregating probability distributions that is provably robust to a specified number of arbitrary constituents. Its risk is controlled by the safest \(s\) input distributions, while \(R\)-robustness also bounds information leakage and can be realized efficiently with bounded abstention.
More broadly, these results open the door to a theoretical treatment of resilience in model outputs: not as a question of interpretability or semantics, but as a property of aggregating a collection of distributions while limiting the influence of unsafe constituent models. This perspective yields formal guarantees that complement empirical oversight methods.

\paragraph{Limitations.} Our work does not provide an end-to-end safety guarantee for several reasons.
First, assuming that at least \(s\) of the \(k\) distributions are truly safe is neither trivial nor inherently justified.
Second, repeated sampling from our algorithm may lead to an accumulation of leakage or unsafe responses.
A third central limitation is our reliance on multiple models and overlap: guaranteeing both safety and non-abstention requires a majority of safe models that overlap. Finally, many risks, such as societal harms measured by distributional disparities \citep{KearnsRoth2019EthicalAlgorithm}, cannot be modeled as a set of unsafe outputs. These guarantees aim to reduce undetectable harms, but techniques that make highly capable systems safer may also make their deployment more likely, with broader costs such as labor displacement that our framework does not address.

\paragraph{Future directions.} Consensus sampling raises questions such as how to extend it to multiple (possibly agentic) interactions.
Although generative AI motivates us, the core problem is more general: consensus sampling aggregates probability distributions given only sampling and probability-evaluation access. Due to this black-box statistical nature, it can apply to future models that may operate in ways quite different from today's models, as long as they exhibit sufficient overlap. A promising future direction is to \emph{engineer} greater overlap (see \Cref{sec:engoverlap}). The copyright-protection literature offers a useful precedent: the stringent overlap-based framework of \citet{vyas2023provable} was followed by more practical adaptive fusion work such as CP-Fuse \citep{abad2025cpfuse}. We hope consensus sampling plays a similar role for safety by isolating the aggregation primitive and the barrier future work should address. In some sense, this quest for overlap is one of identifying a \textit{canonical distribution} which the models can be steered toward.

\begin{ack}
We are grateful to Anastasiia Struss for correcting a mistake in the proof of \Cref{thm:robustness}. We also thank Cary Bassin, Mia Glase, and Min Jae Song for helpful comments. 
Or Zamir's research is supported by an AISI Alignment Project Grant.
\end{ack}


\bibliographystyle{plainnat}
\bibliography{main}

\appendix

\section{Median-distribution warm-up}\applabel{sec:med}
The pointwise median gives a simple intuition for why agreement can suppress unsafe outputs. Define
$$p_{\med}(y) \propto \text{median}\bigl(p_1(y), p_2(y), \ldots, p_k(y)\bigr).$$
It is not clear how to efficiently sample from this distribution. However, if a majority of the distributions are safe, its risk can be bounded in terms of their overlap. For a set $S\subseteq[k]$ of safe distributions, define
$$\Delta(S)\defeq \sum_y \min_{i \in S} p_i(y) \in [0,1].$$
For identical distributions, $\Delta(S)=1$. For $|S|{=}2$, it is the Weizmann overlap, equal to one minus total variation distance; the generalization to arbitrary size is standard \citep[see, e.g.,][]{eidous2025generalization}.

\begin{theorem}[Median distribution safety]\label{thm:med}
    Let $S\subseteq [k]$ be of size $|S|=s>k/2$ with $\Delta(S)>0$.
    For any set of unsafe outputs $U\subseteq Y$,
    $$p_{\med}(U) \le \frac{1}{\Delta(S)}\sum_{i \in S} p_i(U).$$
\end{theorem}
If each safe distribution has $p_i(U)\le\eps$, then $p_{\med}$ has risk at most $\eps s/\Delta(S)$; thus the bound is useful when safe overlap is large relative to safe risk. The theorem also applies to steganography by taking $U=f^{-1}(m)$ for a decoder $f$ and target message $m$.
%
\begin{proof}
Let $d(y)\defeq \med\{p_1(y),\ldots,p_k(y)\}$ and $D\defeq\sum_y d(y)$, so $p_{\med}(y)=d(y)/D$. Since $|S|>k/2$,
\begin{equation}\label{eq:yay}
\min_{i\in S} p_i(y) \le d(y) \le \max_{i \in S} p_i(y).
\end{equation}
Using $\max_{i\in S}p_i(y)\le \sum_{i\in S}p_i(y)$,
\begin{equation}\label{eq:num}
p_{\med}(U) \le \frac{1}{D}\sum_{y\in U}\sum_{i\in S}p_i(y)
= \frac{1}{D}\sum_{i\in S}p_i(U).
\end{equation}
Summing the left side of \Cref{eq:yay} over $y$ gives $\Delta(S)\le D$, which proves the claim and also shows that $p_{\med}$ is well-defined.
\end{proof}
The median distribution motivates the efficient algorithm but has three shortcomings: it is not known how to sample from it efficiently, it requires overlap knowledge about the unknown safe set, and it needs a safe majority. Consensus sampling replaces it with an efficient mechanism whose robustness guarantee requires neither overlap knowledge nor a safe majority, abstaining when needed.


\section{Designing for overlap}\applabel{sec:engoverlap}
We have said little about how $k$ distributions can be trained to encourage both safety and overlap. Intuitively, models trained differently--by different groups or with different methodologies--are more likely to fail independently; such diversity reduces the chance of shared points of failure.

However, independently trained models may exhibit little overlap, especially for long generations---as found in experiments by \citet{vyas2023provable} even when using identical training methodologies. At first it may seem intuitive that one would have exponentially small overlap for long generations, e.g., even if two $\LM$s differ on which type of punctuation they use more often, these differences can accumulate over a long generation. However, a good $\LM$ may have adaptive generation in the sense that given that the first half of a generation uses a certain punctuation convention, it may be likely to continue that convention regardless of its likelihood of using it in the first place. 

We view low overlap in current models as the next stage of the research program rather than a reason to abandon overlap-based methods. The copyright-protection literature provides a useful precedent: \citet{vyas2023provable} introduced a stringent distributional framework, and later work developed more practical adaptive fusion methods within that line \citep{abad2025cpfuse}. Analogously, our contribution is to isolate a probability-aggregation primitive and its guarantees; the next question is how to train, prompt, or compose model families that inhabit a useful high-overlap regime.

Overlap is also more natural in constrained response regimes than in free-form long generation: for binary decisions, agreement and overlap coincide exactly, and tightly specified outputs such as chemical formulas or factual identifications may admit substantial overlap even when the answer space is not known in advance. There are a variety of ways to achieve varied models that still have overlap. For an $\LM$, for instance, one could vary its system instructions to try and thwart various types of harmful generations. A more general methodology, however, is to train or compose models that aim for some \textit{canonical} output distribution for each prompt. This may leverage a small not-so-capable model, $p_\text{small}$, which we believe is unlikely to inject hard-to-detect harmful outputs such as steganography.
\begin{itemize}
    \item Use different large models to generate short ``hints'' which can then be provided to $p_\text{small}$. 
    \item In addition to generating hints, use the large models as ``gatekeepers'' at the end of the generation to determine whether $p_\text{small}$'s output is valid.
    \item Alternatively, the canonical distribution could be $p_\text{small}(\cdot \mid y \in V)$ conditioned on the output $y$ belonging to some set $V$ of good responses. For example, if the prompt is to prove a mathematical theorem, the set $V$ may be the set of valid proofs, which $p_\text{small}$ would be unlikely to generate on its own. This is a bit similar in spirit to maximum-entropy $\LM$s \citep{Rosenfeld1996MaxEntLM} but where the constraints are based on a complex validity criterion.
    \item Finally, $\LM$s could be trained with an objective of staying close to $p_\text{small}$ but with other validity objectives.
\end{itemize}
The requirements for a generative AI oracle to be used with our algorithms is: (a) the ability to sample outputs, conditioned on a prompt; and (b) the ability to compute the conditional probability of any generation $y\in Y$ given a prompt. However, one can extend the framework to allow the sampling oracle to fail and output $\bot$ as well, in which case the algorithm will necessarily resample. This does not impact the safety or abstention guarantees as long as overlap is only computed over $y \in Y$ (and not $\bot$). Also, note that the oracle does not need to be able to compute its abstention probability $p(\bot)$. 

We note that identifying canonical distributions and using them for safety is an interesting future direction of work.

Finally, one could of course run the algorithm with a significantly larger number of iterations $R$, say scaling up by a factor of $2^L$, under the logic that the worst-case bounds in this paper are pessimistic, and in practice an adversary who does not have access to the other models cannot leak many bits. This is similar to how differential privacy \citep{dwork2006calibrating} is often found to be privacy preserving with a much larger tolerance $\eps$ than justified by the worst-case theoretical bounds. Analogous to what \citet{vyas2023provable} do in the context of copyright protection, one can  augment \Cref{alg:efficient} with additional ``slack'' parameter $L\ge 0$ and changing the acceptance probability in line \ref{step:acceptance} to:
        $$\min\left\{1,2^L\frac{\frac{1}{s}\sum_{i\le s} p_{(i)}(y)}{\frac{1}{k}\sum_{i=1}^k p_i(y)}\right\}.$$
This variation is $R'=2^L R$-robust and thus leaks an additional $L$ bits, but abstains exponentially less often. The justification for such a large robustness parameter is that, if overlap is extremely challenging to achieve via canonical distributions, then an adversary who does not have knowledge of the safe distributions may inherently have low overlap. This would be substantially faster than just running the algorithm for $R'$ iterations.

\section{Optimality}\applabel{sec:minmax}
In this section, we show that \Cref{alg:efficient} is optimal in terms of robustness. Recall that we have fixed $s \le k$. However, the algorithm still has a parameter $R=0,1,2,\ldots$ and the abstention rate decreases in $R$ while the robustness increases in $R$. Specifically, we prove that for each parameter setting $R$ the algorithm trades off abstention rate for worst-case robustness in a Pareto optimal fashion. 

The high-level structure of the argument is quite simple. Recall from \Cref{lem:exact} that, for any $R$, the output distribution of the algorithm $q$ is either $\bot$ with probability $q(\bot)=(1-Z)^R$ and $y$ with probability $q(y) = (1-q(\bot))\jinx(y)$. Thus, it is simply a downscaled version of $\jinx$ from \Cref{eq:jinx}. 

Moreover, we show a linear relationship between (non)abstention and robustness, so one can simply trade off abstention for robustness. Finally, we argue that the distribution $\jinx$ is the optimal distribution in terms of worst-case robustness over distributions that do not abstain. This means that each input $R=0,1,2,\ldots$ (and $R=\infty$ corresponding to never abstaining) is Pareto optimal in terms of their abstention and worst-case robustness rates.

In this section, for optimality, we define the smallest value of $R \ge 0$ such that $q$ is $R$-robust with respect to $\pp$, and also the worst-case value of this over all $S$. To do so, it is helpful to define the average distribution $\avgp{S} \in \Distr{Y}$ with respect to any nonempty set of indices $S\subseteq[k]$:
\begin{equation}\label{eq:avgp}
  \avgp{S}(y)\defeq \frac{1}{|S|}\sum_{i \in S}p_i(y) \text{ for all } y \in Y  
\end{equation}
\begin{definition}\label{def:robust2}
    For distribution $q$ over $Y \cup \{\bot\}$ and distributions $p_1,p_2,\ldots, p_k$, over $Y$, 
    $$R_q \defeq \max_{y \in Y}\max_{S \subseteq [k]: |S|=s} \frac{q(y)}{\avgp{S}(y)},$$
    with $0/0$ taken to be 0 and $R_{q}=+\infty$ if $q(y)>0=\avgp{S}(y)$ for some $y\in Y$ and size-$s$ $S$.
\end{definition}
It is easy to see that $q$ is $R'$-robust relative to $p_1,\ldots,p_k$ iff $R'\ge R_q$. $R_q$ is thus the \textit{worst-case} robustness of distribution $q$ with respect to any safe subset. This definition suggests a notion of optimality, for any given $p_1,p_2,\ldots, p_k$, meaning that the algorithm should be optimally robust against the worst-case subset of safe distributions.

First, we show that $\jinx$ is optimal among distributions that never abstain, $\Distr{Y}$.
\begin{lemma}[Minmax robustness]\label{lem:jinx}
For any $\pp=(p_1,p_2,\ldots, p_k) \in \Distr{Y}^k$, 
$$
\min_{q \in \Distr{Y}} R_q = \frac{1}{Z} \text{ for } Z \defeq \sum_y\frac{1}{s} \sum_{i=1}^s p_{(i)}(y),
$$
with $\jinx$ from \Cref{eq:jinx} being the unique minimizer if $Z\ne 0$.
\end{lemma}
\begin{proof}
Fix $\pp$ and hence $Z$. First consider the case that $Z=0$ and let $q \in \Distr{Y}$ be arbitrary. We must show that $R_{q}=\infty$ for some size-$s$ $S$. Let $y \in Y$ be any output with positive probability $q(y)>0$. Since $Z=0$, this means that $\sum_{i \le s}p_{(i)}(y)=0$ so there are $s$ indices $i$ such that $p_i(y)=0$. Let $S$ be this set, then $\avgp{S}(y)=0$ and hence $R_{q}=+\infty$.

Next, consider the case where $Z>0$. We first claim that, $R_{\jinx} \le 1/Z$, equivalently that $\jinx(y)\le \avgp{S}(y)/Z$ for every $|S|=s$. To see this, note that,
  $$\jinx(y) = \frac{1}{Z} \cdot\frac{1}{s}\sum_{i=1}^s p_{(i)}(y) \leq \frac{1}{Z} \cdot \frac{1}{s}\sum_{i \in S} p_i(y) = \frac{\avgp{S}(y)}{Z}.$$
  Conversely, any distribution $q \ne \jinx$ must satisfy, for some $y \in Y$, $q(y) > \jinx(y).$
  Fix such a $y$ and consider the set $S$ which consists of the $s$ distributions with smallest values $p_i(y)$ (breaking ties arbitrarily). For this set $S$, we have,
  $$q(y) > \jinx(y) = \frac{1}{Z} \cdot\frac{1}{s}\sum_{i=1}^s p_{(i)}(y) = \frac{1}{Z} \avgp{S}(y).$$
  Thus, $R_q \ge R_{q,\avgp{S}} > 1/Z$.
\end{proof}

For any distribution $q \in \Distr{Y \cup \{\bot\}}$, it will be convenient in this section to use the term \textit{output rate} for the probability $q(Y)= 1-q(\bot)$ that a distribution $q \in \Distr{Y \cup \{\bot\}}$ produces a non-abstaining output. There is a trivial linear relationship between $q(Y)$ and $R$-robustness which means that there is a single non-abstaining distribution $\jinx$ characterizes the set of Pareto-optimal frontier of distributions in terms of abstaining and robustness.

Any possibly abstaining distribution $q$ can be decomposed into its output rate $q(Y)$ and its \textit{conditional generation distribution} $q'\in \Distr{Y}$ where $q'(y)=q(y)/q(Y)$ for all $y \in Y$. Robustness trivially scales linearly with output rate:
\begin{lemma}[Robustness-abstention tradeoff]\label{lem:tradeoff}
    Let $q\in \Distr{Y \cup \{\bot\}}$ and let $q'\in \Distr{Y}$ be such that $q(y) = q(Y) \cdot q'(y)$ for all $y\in Y$. Then $R_q = q(Y) R_{q'}$.
\end{lemma}
\begin{proof}
    By \Cref{def:robust2}, 
    $$R_{q} = \max_{y \in Y}\max_{S \subseteq [k]: |S|=s} \frac{q(y)}{\avgp{S}(y)} = \max_{y \in Y}\max_{S \subseteq [k]: |S|=s} \frac{q(Y) \cdot q'(y)}{\avgp{S}(y)}= q(Y) R_{q'}.$$
\end{proof}

By this linearity and \Cref{lem:jinx}, it follows immediately that $\jinx \in \Distr{Y}$ characterizes the optimal tradeoff between abstention rate and robustness: a distribution $q$ is Pareto optimal (optimal abstention rate given robustness, and vice versa) if and only if $q(y) = q(Y) \jinx(y)$ for all $y \in Y$. 
\begin{theorem}[Optimality]\label{thm:pareto}
Fix any $\pp=(p_1,\ldots,p_k)\in\Distr{Y}^k$ and $R\in\{0,1,2,\ldots\}\cup\{\infty\}$. Let $q$ be the (possibly abstaining) output distribution of \Cref{alg:efficient} with parameter $R$. Then $q$ is Pareto optimal: for every other (possibly abstaining) distribution $\nu \ne q$, either $q(\bot)<\nu(\bot)$ or $R_q< R_\nu$ (or $R_\nu=\infty$).
\end{theorem}
\begin{proof}
Fix a distribution $\nu$ such that $\nu(\bot)\le q(\bot)$ and $R_\nu \le R_q$. We will show that $\nu=q$. 

Let $\jinx$ and $Z$ be as defined in \Cref{lem:jinx}. If $Z=0$, then Algorithm~\ref{alg:efficient} always abstains, so $q(\bot)=1$ and $R_q=0$, any any $\nu$ with $R_\nu=0$ must always abstain. Thus, we can assume $Z>0$ and $\jinx$ is well-defined as in \Cref{eq:jinx}, and $R_\jinx < \infty$. By the linearity of $R$ (\Cref{lem:tradeoff}) and the fact that $\jinx$ is the unique distribution with $\jinx(Y)=1$ and $R_\jinx \le 1/Z$ (\Cref{lem:jinx}), it is easy to see that the only Pareto optimal distributions are $\mu$ which satisfy $\mu(y)=\mu(Y) \cdot \jinx(y)$ for all $y \in Y$, i.e., the distributions $\mu$ which abstain with some probability $\mu(\bot)$ and with the remaining probability $\mu(Y)=1-\mu(\bot)$ output according to $\jinx.$ (To see this, just consider $\nu^* \in \Distr{Y}$ with $\forall y\in Y\,\nu^*(y)\defeq \nu(y)/\nu(Y)$, and note that if $\nu^*\ne q^*$ then $R_{\nu^*}>1/Z$.) Finally, \Cref{lem:exact} shows exactly that: $q(y)=q(Y)\jinx(y)$ for all $y \in Y$.
\end{proof}

\section{Mutual information leakage bound}\applabel{sec:mutinfo}

In this section, we observe that $R$-robustness also implies a quantitative bound on information leakage: at most $\log(R+1)$ bits can be leaked by an adversarial robust distribution, beyond what is already leaked by the safe distribution. This limits the impact of steganography or any attempt to smuggle information external to the safe distribution. \Cref{sec:prompts} will consider the more general case where the prompt is not fixed.

It will be convenient to define a related notion of safety with respect to a \textit{single safe distribution} $p$. 
\begin{definition}[Relative risk]
    For $R \ge 0$, $q \in \Distr{Y \cup \{\bot\}}$, and $p \in \Distr{Y}$, say $q$ is $R$-risky relative to $p$ if $q(y) \le R \cdot p(y)$ for all $y \in Y$.
\end{definition}
If $q$ is $R$-robust, then it follows directly that, for any size-$s$ $S \subseteq [k]$, $q$ is $R$-risky with respect to $\avgp{S}$ which is defined in \Cref{eq:avgp} by $\avgp{S}(y)\defeq \frac{1}{|S|}\sum_{i \in S}p_i(y).$ 
\begin{observation}
  $q$ is $R$-robust relative to $p_1,\ldots,p_k$ if and only if $q$ is $R$-risky relative to $\avgp{S}$ for every $S \subseteq [k]$ of size $s$.
\end{observation}
\begin{proof}
($\Rightarrow$) Take $U=\{y\}$, so $q(y)\le R\cdot \tfrac1s\sum_{i\le s}p_{(i)}(y)\le R\,\overline p_S(y)$ for any $|S|=s$. 
($\Leftarrow$) Summing $q(y)\le R\,\overline p_S(y)$ over $y\in U$ gives $q(U)\le R\,\overline p_S(U)$ for all $S$; minimizing over $S$ yields $q(U)\le R\cdot \tfrac1s\sum_{i\le s}p_{(i)}(U)$.
\end{proof}

Here we model leakage as follows: the safe distribution $p=\avgp{S}$ is fixed while there is a distribution $q_m$ which varies based on some message $m$ to be encoded, which is drawn from a message distribution~$M$. 
We show that $q$ being $R$-risky implies that the mutual information between the drawn message $M$ and the response sample $Q\sim q_M$ is bounded by $\log R+1$. In particular, at most $\log R+1$ bits beyond what is already present in $p$ can be encoded in this way --- using steganography or otherwise.

Recall that for a distribution $p\in \Distr{Y}$, its (Shannon) entropy is
 \[
 H(p) \defeq -\sum_{y\in Y} p(y)\log p(y).
 \]
 For jointly distributed random variables $Q,M$, recall the mutual information is
 \[
 I(Q;M) \defeq H(Q)+H(M)-H(Q,M) = H(Q)-H(Q\mid M).
 \]

\begin{theorem}[Mutual-information leakage]\label{thm:mutinfo}
Fix $p\in \Distr{Y}$ and $R\ge 0$. Let $M$ be any message random variable.
For each $m$ in the support of $M$, let $q_m\in \Distr{Y\cup\{\bot\}}$ be $R$-risky relative to $p$,
i.e., $q_m(y)\le R\,p(y)$ for all $y\in Y$.
Let $Q\sim q_M$.
Then $I(Q;M)\le \log(R+1)$.
\end{theorem}

\begin{proof}
Let $Y'\defeq Y\cup\{\bot\}$ and define $p'\in\Distr{Y'}$ by
$p'(\bot)=\frac{1}{R+1}$ and $p'(y)=\frac{R}{R+1}\,p(y)$ for $y\in Y$.
Then for all $y\in Y'$, $q_m(y)\le (R+1)\,p'(y)$ (since $q_m(\bot)\le 1$ and $q_m(y)\le R p(y)$ on $Y$).

Hence, pointwise $\log q_m(y)\le \log(R+1)+\log p'(y)$, so
\[
H(q_m)=-\!\sum_{y\in Y'} q_m(y)\log q_m(y)
 \ge  -\!\sum_{y\in Y'} q_m(y)\log p'(y) - \log(R+1).
\]
Averaging over $M$ gives
\[
H(Q\mid M) \ge  -\sum_{y\in Y'} \Pr[Q=y]\log p'(y) - \log(R+1).
\]
By the cross-entropy inequality (with the same support $Y'$),
\[
H(Q) \le  -\sum_{y\in Y'} \Pr[Q=y]\log p'(y).
\]
Subtracting the inequalities yields $I(Q;M)=H(Q)-H(Q\mid M)\le \log(R+1)$.
\end{proof}

\section{Varying prompts}\applabel{sec:prompts}
In this final section, we consider the inclusion of prompts $x \in X$. A model is now a conditional generation distributions $p \in \Distr{Y \mid X}$. The only change to the algorithm is the inclusion of the prompt $x \in X$ as an input, shown in \Cref{alg:prompts}. 

\begin{algorithm}[tb]
  \caption{Consensus sampling with prompts}
  \label{alg:prompts}
  \begin{algorithmic}[1]
    \Require Prompt $x \in X$, integers $k \ge s \ge 1$, $R \ge 0$, distributions $\pp \in (\Distr{Y\mid X})^k$
    \Ensure A sample $y\in Y$ or $\bot$
    \For{$r = 1$ to $R$}
      \State Pick a random index $i$ and sample $y \sim p_i(\cdot \mid x)$
      \State With probability $\dfrac{\frac{1}{s}\sum_{i\le s} p_{(i)}(y \mid x)}{\frac{1}{k}\sum_{i=1}^k p_i(y \mid x)}$ \textbf{return} $y$  
    \EndFor
    \State \Return $\bot$
  \end{algorithmic}
\end{algorithm}

\Cref{obs:steg} trivially generalizes to using a single decoding function across prompts. (One can also generalize to a decoder $f: X \times Y \rightarrow \{0,1\}^*$.)
\begin{observation}\label{obs:stegprompt}
Fix a decoding function $f: Y \rightarrow \{0,1\}^*$ and $R \ge 1$. Let $q_x, p_x \in \Distr{Y}$ be output and safe distributions over $y$, given prompt $x$. If, for every prompt $x$, $q_x$ is $R$-robust relative to $p_x$, then:
$$\Pr_{Q \sim q_x}[f(Q)=m] \le R \cdot \Pr_{P \sim p_x}[f(P)=m] \text{ for every }x \in X, m \in \{0,1\}^*.$$
\end{observation}
The proof again follows trivially from the definition of $R$-robustness by considering the unsafe sets $U=f^{-1}(m)$. Unfortunately, the direct analog of the mutual information bound \Cref{thm:mutinfo} does not hold in the setting where the safe distribution depends on a prompt $x$. Instead, we use a more refined model of leakage called maximal leakage~\citep{issa2020maxleakage} in place of mutual information. 

We consider a generation process that may involve a prompt $x$ which determines the safe distribution $p_x\in \Distr{Y}$. That is, for each prompt $x$ we assume there is a fixed distribution $p_x$. We are concerned about leakage from a message $m$, which is possibly related to the prompt $x$. To this end, suppose that for each $m,x$ there is a distribution of responses $q_{x,m}\in \Distr{Y}$. We thus wish to know how much information from $m$ leaks into a sample $y \sim q_{x,m}$. Of course, significant information from the prompt may already be naturally encoded into $p_x$, so the question is then how much \textit{more information} may leak from $m$ into $q_{x,m}$ than leaks from $m$ into $p_x$. 

Issa et al.~\citep{issa2020maxleakage} define and justify the following definition of \textit{maximal leakage} as a better measure of leakage than mutual information which, as they discuss, has known flaws.
\begin{definition}[Maximal leakage \citep{issa2020maxleakage}]\label{def:maxleak}
For random variables $M,Y$, the maximal leakage from $M$ to $Y$ is
\[
\mathcal{L}(M \to Y)\ \defeq\ \log \sum_y\max_m \Pr[Y=y\mid M=m].
\]
\end{definition}
If $Y$ and $M$ are independent, then $\mathcal{L}(M \to Y)=0$. If $Y=M$ then $\mathcal{L}(M \to Y)$ is log of the number of possible messages.

\begin{theorem}[Maximal leakage]\label{thm:leakage}
Let $X, M$ be jointly distributed random variables (for the prompt and message, respectively). 
For each of their possible values $(x,m)$, let $p_x\in \Distr{Y}$ be a safe distribution and $q_{x,m}\in \Distr{Y \cup \{\bot\}}$ a sampling distribution which is assumed to be $R$-robust relative to $p_x$. 
Define $Q$ by sampling $Q\sim q_{X,M}$ and $P$ by sampling $P\sim p_X$. Thus, conditioned on $X$, $P$ and $M$ are independent.
Then
\[
\mathcal{L}(M \to Q) - \mathcal{L}(M \to P) \le \log(R+1).
\]
\end{theorem}
  \begin{proof}[Proof of \Cref{thm:leakage}]
    For any $y\in Y,m \in \{0,1\}^*$, since $q_{x,m}$ is $R$-robust relative to $p_x$ for all $x$,
    \begin{align*}
    \Pr[Q=y\mid M=m] &= \E[q_{X,m}(y)\mid M=m]\\
    &\le \E[R \cdot p_X(y)\mid M=m]\\
    &= R \cdot \Pr[P=y\mid M=m].
    \end{align*}
    Thus, for any $y \in Y$,
    $$
    \max_m \Pr[Q=y\mid M=m] \le \max_m R \cdot \Pr[P=y\mid M=m].
    $$
    Hence, summing over $y$ and taking logs gives,
    \begin{align*}
      \mathcal{L}(M \to P) &= \log \sum_{y \in Y \cup \{\bot\}} \max_m \Pr[Q=y\mid M=m]\\
      &\le \log \left(1+\sum_{y \in Y} \max_m \Pr[Q=y\mid M=m]\right)\\
      &\le \log \left(1+\sum_y \max_m R \cdot \Pr[P=y\mid M=m]\right) \\
      &= \log \sum_y \left(\Pr[P=y] + R \cdot \max_m \Pr[P=y\mid M=m]\right) \\
      &\le \log \sum_y ( R+1) \cdot \max_m \Pr[P=y\mid M=m] \\
      &=\log(R+1) + \log \sum_y \max_m \Pr[P=y\mid M=m]\\
    &=\log(R+1) + \mathcal{L}(M \to P)
    \end{align*}
    This is exactly the bound $\mathcal{L}(M \to Q) \le \log(R+1) + \mathcal{L}(M \to P)$ needed for the theorem.
    \end{proof}


\end{document}